\documentclass[lettersize,journal]{IEEEtran}

\usepackage[caption=false,font=normalsize,labelfont=sf,textfont=sf]{subfig}
\usepackage{cite}
\usepackage{multicol}
\usepackage[bookmarks=true]{hyperref}

\usepackage{bm}
\usepackage{amssymb}
\usepackage{dsfont}
\usepackage{amsmath,amsfonts}
\usepackage{titlesec}
\usepackage{indentfirst}
\usepackage{algorithm}
\usepackage{algorithmicx}
\usepackage{algpseudocode}

\usepackage{amsthm}
\usepackage{graphicx}
\usepackage{url}
\usepackage{svg}
\usepackage{nicefrac}
\usepackage{balance}
\usepackage{mathtools}
\usepackage{booktabs}
\usepackage{siunitx}
\usepackage{multirow}
\usepackage{pgfplots}
\usepackage{pgfplotstable}
\usepackage{tikz}
\usepackage{placeins}
\pgfplotsset{compat=newest}
\usepgfplotslibrary{groupplots}
\usepgfplotslibrary{units}
\usepgfplotslibrary{statistics}
\usepackage{subcaption}
\usepackage[percent]{overpic}
\definecolor{customred}{RGB}{220,33,77}
\definecolor{lightred}{RGB}{255,150,150}
\definecolor{customblue}{RGB}{0,100,222}
\definecolor{lightblue}{RGB}{150,200,255}

\usepackage{cleveref}
\Crefname{figure}{Fig.}{Figs.}

\newtheoremstyle{exampstyle}
  {3pt} 
  {3pt} 
  {\itshape} 
  {} 
  {\bfseries} 
  {.} 
  {.5em} 
  {} 

\theoremstyle{exampstyle} 
\newtheorem{definition}{Definition}
\newtheorem{corollary}{Corollary}

\newtheorem{theorem}{Theorem}

\newtheorem{assumption}{Assumption}

\newtheorem{example*}{Example*}

\theoremstyle{plain}

\newcommand{\RNum}[1]{\uppercase\expandafter{\romannumeral #1\relax}}

\usepackage[table,RGB,dvipsnames]{xcolor}
\usepackage[normalem]{ulem} 

\algtext*{EndFor}
\algtext*{EndIf}

\newif\ifrevised
\revisedfalse  
\newcommand{\rev}[1]{\ifrevised{\color{blue}#1}\else#1\fi}
\newcommand{\del}[1]{\ifrevised{\textcolor{red}{\sout{#1}}}\else{}\fi}

\begin{document}

\title{Safe and Scalable Multi-Drone Payload Transport via CBF-based Reinforcement Learning with Zero-Shot Sim-to-Real Transfer}

\author{
Jaeyoun Choi\textsuperscript{1},
Oswin So\textsuperscript{1},
Songyuan Zhang\textsuperscript{1},
Cooper Taylor\textsuperscript{2},
and Chuchu Fan\textsuperscript{1}
%
%
%
\thanks{\textsuperscript{1}Jaeyoun Choi, Oswin So, Songyuan Zhang, and Chuchu Fan are with the Reliable Autonomous Systems Lab, Massachusetts Institute of Technology, Cambridge, MA, USA.
Emails: {\tt\small \{cjy0051, oswinso, szhang21, chuchu\}@mit.edu}.}
\thanks{\textsuperscript{2}Cooper Taylor contributed to this work while participating in the Research Science Institute program at the Reliable Autonomous Systems Lab, Massachusetts Institute of Technology, Cambridge, MA, USA.
Email: {\tt\small cooperetaylor@gmail.com}.}
\thanks{The authors were partly funded by the Ministry of Trade, Industry and Energy (MOTIE), Korea, through the Global Industrial Technology Cooperation Center (GITCC) Program, supervised by the Korea Institute for Advancement of Technology (KIAT), Task No. P24680172.}
%
}
\markboth{IEEE Robotics and Automation Letters. Preprint Version. Accepted July, 2026}
{Choi \MakeLowercase{\textit{et al.}}: Safe and Scalable Multi-Drone Payload Transport via CBF-based Reinforcement Learning with Zero-Shot Sim-to-Real Transfer}
\maketitle
\IEEEaftertitletext{\IEEEpubid}

\begin{abstract}
Multi-drone payload transportation has emerged as a promising research paradigm with potential applications in construction, logistics, and disaster response. However, the complex coupled dynamics among drones, cables, and payloads pose significant challenges, and existing approaches remain limited in safety and scalability, particularly in dynamic and unstructured environments. In this work, we propose a learning-based framework for safe and scalable multi-drone cooperative payload transport. We introduce a minimal 2D abstraction that preserves the task-relevant drone–payload coupling required for coordination and safety, while remaining computationally efficient for large-scale learning. Using domain randomization over team size and physical parameters, we train a fully distributed policy via Discrete Graph Control Barrier Function Proximal Policy Optimization (DGPPO), enabling robust zero-shot sim-to-real transfer without fine-tuning. Extensive real-world evaluations demonstrate that a single learned policy generalizes across varying team sizes and task scenarios. Furthermore, multi-group hardware experiments show that the same policy can safely operate in dynamic environments, where other drone teams act as moving obstacles. These results indicate that the proposed framework enables efficient, safe, and scalable multi-drone payload transportation with strong generalization to complex real-world conditions.
\end{abstract}

\begin{IEEEkeywords}
Multi-Robot Systems, Robot Safety, Aerial Systems: Applications
\end{IEEEkeywords}



\section{Introduction}
\IEEEPARstart{C}{ooperative} payload transportation is a central problem in multi-robot systems, where multiple robots physically coordinate to move a shared payload~\cite{kube2000cooperative,tuci2018cooperative,an2023multi,arbel2025mechanical,alonso2017multi,de2022two}. One prominent example is cable-suspended multi-drone systems, which offer structural simplicity, modularity, and adaptability~\cite{barakou2024review}. These systems are attractive for practical applications, including last-mile delivery, disaster relief, automated construction, and logistics.

Scalability is essential because larger teams can improve payload capacity, load distribution, and deployment flexibility~\cite{barakou2024review,jackson2020scalable,li2023nonlinear}. However, scaling beyond a few drones remains challenging in real-world settings due to coupled cable-payload interactions, inter-agent interactions, sensing noise, and hardware heterogeneity~\cite{barakou2024review, jackson2020scalable,li2023nonlinear}. Consequently, most real-world demonstrations remain limited to two or three drones.

Safety further compounds this challenge. Prior work has demonstrated cooperative aerial transport using centralized optimization, distributed trajectory generation, or decentralized sensing~\cite{alonso2017multi,lee2016planning,jackson2020scalable,wehbeh2020distributed,gassner2017dynamic}. However, real-world validation has largely remained limited to small teams and structured environments, while safety is often enforced through \textit{centralized} planners that rely on explicit environment representations~\cite{barakou2024review,li2023nonlinear,sarvaiya2024hpa,li2024human2}. Achieving scalable payload transport with reactive environment-level safety under fully \textit{distributed} execution, therefore, remains largely open.
\begin{figure*}[!t]
    \centering
    \includegraphics[width=\textwidth]{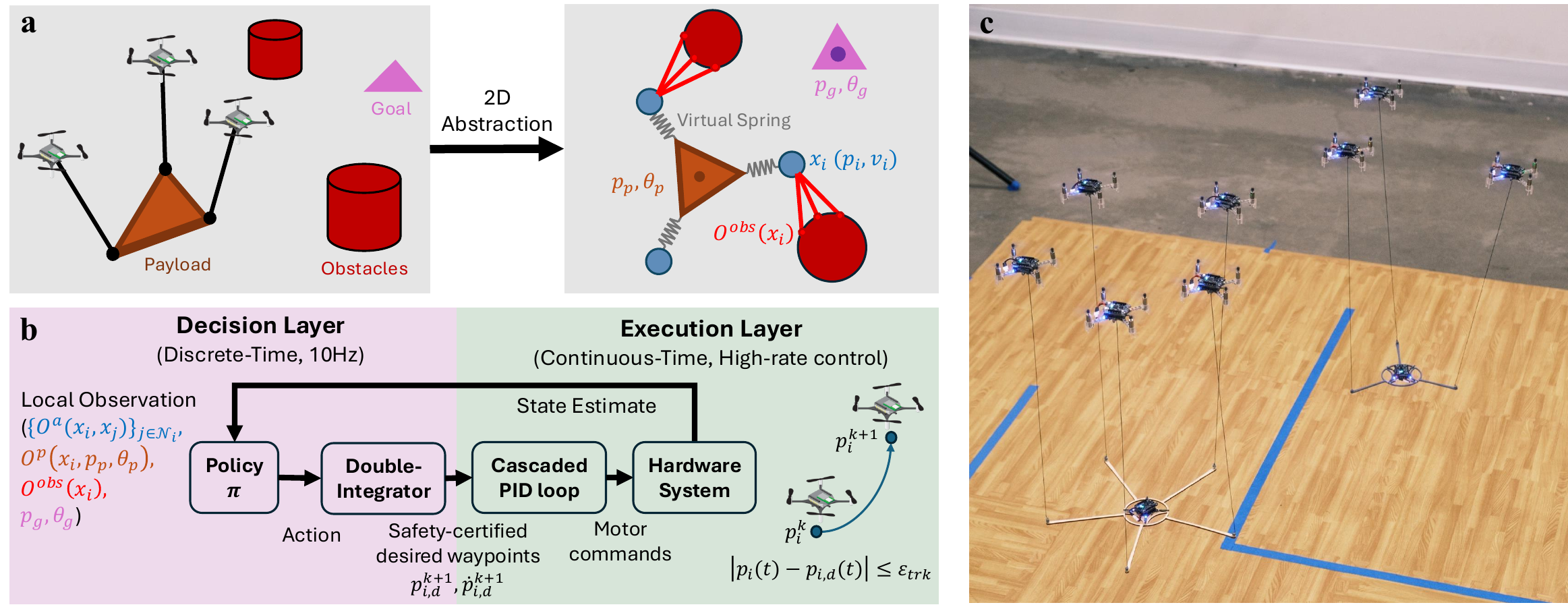}
    \vspace{-0.4cm}
    \caption{\textbf{Distributed learning framework for safe and scalable multi-drone transportation.}
    (\textbf{a}) Real-world cable-suspended transport and corresponding
    2D abstraction. Drones are modeled as point masses connected to the payload via virtual springs, and obstacles are detected with LiDAR.
    (\textbf{b}) Control architecture and discrete-to-continuous safety bridge. A distributed policy $\pi$ generates DGCBF-certified waypoints from local observations, tracked by an onboard PID controller. With bounded tracking error $\varepsilon_{\mathrm{trk}}$, tightened DGCBF constraints ensure continuous-time safety.
    (\textbf{c}) Zero-shot real-world deployment. The learned policy transfers directly to physical quadrotors without fine-tuning, demonstrating robust coordination in dynamic environments.
    }
    \vspace{-0.4cm}
    \label{fig:Figure1}
\end{figure*}

To address these challenges, we present a learning-based framework for scalable,
safe cable-suspended multi-drone transportation. To our knowledge, no prior work has demonstrated real-world cable-suspended multi-drone transportation with a safe, \textit{fully distributed} policy in dynamic environments with moving obstacles. Our main \textbf{contributions} are:

\begin{itemize}
  \item \textbf{Efficient simulation and domain randomization for scalable and safe policy learning:}
We develop an efficient simulation environment based on a minimal 2D transport abstraction.
We learn a shared-parameter distributed policy by using domain randomization that generalizes across team sizes and task instances while enforcing safety constraints.

\item \textbf{Bridging learned discrete-time safety to continuous-time robotic execution:}
We provide a safety analysis that connects discrete-time safety guarantees obtained from learned control barrier functions to continuous-time execution under hierarchical tracking control, enabling their practical deployment on real robotic systems.

 \item \textbf{Scalable and safe real-world multi-drone transportation:} We demonstrate reliable payload transport on Crazyflie teams of up to six drones, exceeding the training range, including multi-group hardware experiments in which other drone teams act as dynamic obstacles.
\end{itemize}

\section{Related Work}

\subsection{Control of Multi-Drone Payload Transportation}

A wide range of control and planning methods has been applied to cable-suspended multi-drone transportation systems. Centralized control frameworks have demonstrated precise object transport in simulation or experimental settings~\cite{lee2016planning,li2023nonlinear,wehbeh2020distributed,nguyen2018novel,michael2011cooperative,bernard2010load,sanalitro2020full,sun2023nonlinear,pallar2025optimal,sarvaiya2024hpa,sun2025agile}. While effective for small teams, these methods often rely on global state information and centralized computation, which limits their scalability in real-world hardware demonstrations.

In parallel, decentralized~\cite{sundin2022decentralized,oishi2022cooperative,mohammadi2018decentralized} and distributed~\cite{wehbeh2020distributed,li2021cooperative,gassner2017dynamic,tagliabue2017collaborative,lee2017geometric,wu2014geometric, jackson2020scalable, zhang2025defmarl} control strategies have been explored.
These approaches leverage local sensing and limited communication, and several have demonstrated real-world feasibility with two or three drones. However, most existing methods either rely on leader-based architectures with a single point of failure or are validated mainly in simulation.

Learning-based approaches are another attractive control approach due to their ability to handle nonlinear dynamics and partial observability. However, most prior works focus on simulation-only results, and successful sim-to-real transfer remains limited~\cite{faust2017automated,li2022deep,estevez2024reinforcement,mo2024dral,khursheed2024cooperative,li2021multi}. Zeng et al.~\cite{zeng2025decentralized} recently demonstrated hardware experiments based on decentralized reinforcement learning, but their approach assumes collision-free environments and does not explicitly address safety.

\subsection{Safety in Multi-Robot Transportation Systems}

Safety is critical for real-world multi-drone transportation, especially in cluttered and dynamic environments. Recent safety-aware approaches enforce safety constraints through centralized, model-based optimization~\cite{sun2025agile,jackson2020scalable} or exploit internal system redundancy for human-aware safety under hierarchical control~\cite{li2024human2}. However, these methods typically rely on global state information and centralized computation, which can hinder scalability and complicate deployment in distributed robotic systems.

\section{Problem Formulation}

Our goal is to learn distributed control policies for $N$ drones to collaboratively transport a suspended payload to a target position $p_g$ and orientation $\theta_g$, while avoiding unsafe configurations such as collisions or boundary violations.
We formulate the task as a multi-agent constrained optimal control problem under partial observability. Let the state of drone $i$ at time step $k$ be $x_i^k = (p_i^k, v_i^k)$, where $p_i^k \in \mathbb{R}^2$ denotes the planar position and $v_i^k \in \mathbb{R}^2$ the planar velocity. 
The control input is the commanded horizontal acceleration $u_i^k \in \mathcal{U}$ with $\mathcal{U} := \{u \in \mathbb{R}^2 \mid u_{\min} \le u \le u_{\max}\}.$

The full system state is
\[
\mathbf{x}^k := [x_1^k, \dots, x_N^k, p_p^k, \theta_p^k, v_p^k, \omega_p^k, p_g, \theta_g]\in \mathcal{X},
\]
which includes all drone states as well as the payload's position $p_p^k$, orientation $\theta_p^k$, velocity $v_p^k$, and angular velocity $\omega_p^k$. Denote the joint action by $\mathbf{u}^k := [u_1^k, \dots, u_N^k]$.  The system evolves according to discrete-time dynamics
\[
\mathbf{x}^{k+1} = f(\mathbf{x}^k, \mathbf{u}^k),
\]
where $f$ captures the coupled drone-payload interactions induced by the cable dynamics that will be described in the following section.

Each agent operates under partial observability. Specifically, drone $i$ receives a local observation
\begin{equation}
o_i^k = \left( \{ o_{ij}^k \}_{j \in \mathcal{N}_i},\; o_i^{\text{payload},k},\; o_i^{\text{obs},k},\; p_g, \theta_g\right),
\end{equation} 
where $\mathcal{N}_i = \{j \mid \|p_j^k - p_i^k\| \leq R\}$ denotes the set of nearby agents within sensing radius $R$. The observation components are given by
\begin{equation}
\begin{cases}
o_{ij}^k = \mathcal{O}^a(x_i^k, x_j^k), & \text{(inter-agent features)}, \\
o_i^{\text{payload},k} = \mathcal{O}^p(x_i^k, p_p^k, \theta_p^k), & \text{(payload features)}, \\
o_i^{\text{obs},k} = \mathcal{O}^{\text{obs}}(p_i^k), & \text{(obstacle features)},
\end{cases}
\end{equation}
where $\mathcal{O}^{\text{obs}}$ represents LiDAR-based obstacle observations,
consisting of points on obstacle surfaces detected by each agent within a sensing
radius $R_{Lidar}$.

Safety requirements are encoded as $M$ local constraint functions $\{h_i^{(m)}\}_{m=1}^{M}$ for each agent $i$, where $h_i^{(m)}:\mathcal O \to \mathbb R$ is the $m$-th observation-based constraint.
The corresponding unsafe/avoid set for agent $i$ is
\[
\mathcal{A}_i := \left\{ o_i \in \mathcal{O} \;\middle|\;
\exists m\in\{1,\dots,M\}: \; h_i^{(m)}(o_i) > 0
\right\},
\]
which can represent inter-agent and agent-obstacle collisions, workspace boundary violations, or other task-specific safety limits.

The objective is to learn a distributed policy
$\pi : \mathcal{O} \rightarrow \mathcal{U}$ such that given a cost function $l$ describing a task, the joint policy
$\boldsymbol{\pi}(\mathbf{x}^k) := [\pi(o_1^k), \dots, \pi(o_N^k)]$ minimizes the cumulative task cost while maintaining safety:
\begin{subequations}
\begin{align}
\min_{\pi} & \quad \sum_{k=0}^{\infty}
\ell(\mathbf{x}^k, \boldsymbol{\pi}(\mathbf{x}^k)) \\
\text{s.t.} \quad
& \mathbf{x}^{k+1} = f(\mathbf{x}^k, \boldsymbol{\pi}(\mathbf{x}^k)), \\
& o_i^k = \mathcal{O}_i(\mathbf{x}^k), \quad \forall i, k, \\
& h_i^{(m)}(o_i^k) \leq 0, \quad \forall i, m, k .
\end{align}
\label{eq:problem_formulation}
\end{subequations}

\section{Framework for Scalable and Safe Multi-Drone Transport}
\subsection{2D Abstraction for Multi-Drone Transport}

Modeling cable-suspended multi-drone transportation in full three-dimensional space is challenging due to strongly coupled, hybrid dynamics, including cable slack-taut transitions and unmodeled aerodynamic effects. While high-fidelity models can partially capture these phenomena, they significantly increase computational complexity. Moreover, a recent study demonstrated that modeling fidelity alone does not eliminate the sim-to-real gap in practice~\cite{zeng2025decentralized}. 

In this work, we therefore adopt a minimal 2D abstraction (\Cref{fig:Figure1}a) that preserves the task-relevant coupling between drone configurations and the resulting net force and torque applied to the payload, while abstracting away out-of-plane dynamics that are regulated by a low-level controller. The goal of this abstraction is to retain the decision-relevant interactions necessary for coordination and safety while enabling computationally efficient policy learning.

The abstraction is based on the following assumptions, which hold for tasks in which altitude is regulated and payload motion remains close to planar alignment:

\begin{enumerate}
    \item Drones maintain a regulated altitude, such that vertical motion introduces only bounded disturbances to the planar dynamics.
    \item The tension of each cable is proportional to its 2D projected length, capturing the dominant planar force contribution.
    \item The payload remains approximately planar, with roll and pitch motions negligible compared to the translational and yaw dynamics relevant to the task.
    \item Cable attachment points are fixed and evenly distributed around the payload’s center of mass.
\end{enumerate}

Under these assumptions, the payload is modeled as a planar rigid body governed by
\begin{equation}
M_p \ddot{{p}}_p = \sum_{i=1}^N {T}_i, \qquad
I_p \ddot{\theta}_p = \sum_{i=1}^N {r}_i \times {T}_i,
\end{equation}
where ${p}_p \in \mathbb{R}^2$ denotes the planar position of the payload, $\theta_p$ its yaw angle,  $M_p$ its mass, $I_p$ its moment of inertia about vertical axis, $N$ the number of agents, and ${T}_i \in \mathbb{R}^2$ the planar tension applied at attachment point ${r}_i$, expressed in the payload frame. Under the above assumptions, each cable shares the payload weight through its vertical component, while the planar component of the tension increases with the cable angle relative to the vertical. Specifically, we model the tension vector ${T}_i$ for drone $i$ as a virtual spring pulling toward its designated target location on the payload perimeter:
\begin{equation}
{T}_i = \mathcal{K} \left[ {p}_i - \left( p_p + r \cdot \begin{bmatrix} \cos\left( \frac{2\pi i}{N} \right) \\ \sin\left( \frac{2\pi i}{N} \right) \end{bmatrix} \right) \right],
\end{equation}
where $\mathcal{K}$ is a spring coefficient, ${p}_i$ is the position vector of the $i$th drone, and $r$ is the radial distance from the center of the payload to the cable attachment point. Payload orientation is governed by the planar cable forces, which depend primarily on drone positions rather than drone orientation.

We discretize these dynamics using a double-integrator update with timestep $\Delta t$, which serves as the decision-level model for policy learning and safety certification. Any discrepancy between this abstraction and the continuous-time three-dimensional execution is handled by closed-loop tracking control and explicitly accounted for in the safety analysis in \Cref{app:tracking_dgcbf_proof}.

\subsection{Learning Safe Distributed Policies with DGPPO}
\subsubsection{Discrete Graph Control Barrier Function Proximal Policy Optimization (DGPPO)}
We apply a sim-to-real transfer approach, where the policy is learned in simulation
and then transferred to real-world hardware. DGPPO~\cite{zhang2025discrete} is adopted as a distributed policy optimization backbone
for training multi-drone transport policies under partial observability and safety constraints. Each agent executes a shared-parameter policy $\pi$ using only local graph-based observations, where a gated recurrent unit (GRU) maintains an internal memory to mitigate partial observability.

DGPPO jointly learns a distributed control policy and a constraint value function that serves as a Discrete Graph Control Barrier Function (DGCBF)~\cite{zhang2025gcbf+}. Control Barrier Functions (CBFs) enforce safety by maintaining
control-invariant safe sets~\cite{ames2019control,alan2023control}; however,
classical formulations typically assume centralized access to the full system state.
Such assumptions are impractical in our setting due to complex coupled dynamics and
partial observability. Instead, DGCBFs certify safety directly in the observation
space using only local information, enabling distributed safety assessment and
filtering policy updates that are predicted to violate constraints without requiring
global state access. We refer readers to~\cite{zhang2025discrete} for full algorithmic
details and guarantees, and formalize the continuous-time execution implications in \cref{app:tracking_dgcbf_proof}.

\subsubsection{Domain Randomization}
To promote robustness and generalization, we apply domain randomization during training. Specifically, during each training episode, the number of agents is randomly selected from $\{3, 4, 5\}$. For each case, the agents are symmetrically arranged around the payload with equal radial distance from its center\del{, ensuring balanced tension distribution}.

In addition, we randomize the virtual spring stiffness $\mathcal{K}$, which governs the tension dynamics between each drone and the payload. The stiffness coefficient is uniformly sampled from the range $[0.05, 0.30] \si{\newton\per\meter}$. This range is inspired by the linearized spring approximation of a pendulum, where the effective stiffness can be estimated as $k \approx \frac{mg}{Nl}$ (e.g., for $m \approx 0.027\si{kg}$ , $N = 3$,  $l = 0.5\si{m}$, and $k \approx 0.18 \si{\newton\per\meter}$). By sampling across a broad range around this nominal value, we ensure that the policy remains effective under different tether dynamics and supports smooth sim-to-real transfer.

\section{Continuous-Time Safety under Hierarchical Control}
\label{app:tracking_dgcbf_proof}

Zhang et al.~\cite{zhang2025discrete} have shown that DGCBFs can guarantee safety for multi-agent systems under discrete-time dynamics and decentralized observations. In practice, however, policies learned and certified in discrete time are executed on physical systems that evolve continuously, typically through a hierarchical control architecture in which discrete-time decisions are tracked by a high-bandwidth low-level controller (\Cref{fig:Figure1}b). This section bridges this gap by formalizing conditions under which DGCBF-based safety guarantees extend from discrete-time decision making to continuous-time execution. We begin by recalling the following corollary from prior work~\cite{zhang2025discrete}.
\begin{corollary}[DGCBF]
\label{def:dgcbf}

Define the constrained value function 
$\tilde{V}_{h^{(m)}, \pi} : \mathcal{O} \to \mathbb{R}$,
which depends only on local observations $o_i = O_i(x)$, as
\begin{equation}
\begin{aligned}
\tilde{V}_{h^{(m)}, \pi}(o_i^0)
:= \max_{k \ge 0} h^{(m)}(o_i^k),
\end{aligned}
\end{equation}
Then, the function $\tilde{V}_{h^{(m)}, \pi}$ is DGCBF.

\end{corollary}

\begin{assumption}[Tracking error bound]
\label{ass:tracking_error}
At discrete decision steps $k$, the policy outputs desired accelerations $u_i^k$, which are integrated to generate constant position references $p_{i,d}(t)$ over each decision interval $t\in[k\Delta t,(k+1)\Delta t)$. A low-level feedback controller operating at a higher update rate tracks these references, yielding a continuous-time closed-loop trajectory. There exists a known constant $\varepsilon_{\mathrm{trk}} > 0$ such that, for all agents $i$ and all $t\ge 0$,
\begin{equation}
\|p_i(t) - p_{i,d}(t)\| \le \varepsilon_{\mathrm{trk}} .
\label{eq:tracking_error}
\end{equation}
\end{assumption}

\begin{assumption}[Lipschitz safety constraints]
\label{ass:lipschitz_constraints}
Each local safety constraint function $h_i^{(m)}$ is Lipschitz
in $o_i$ with constant $L_{i,m}\ge 0$, i.e.,
\begin{equation}
|h_i^{(m)}(o_i)-h_i^{(m)}(\bar{o}_i)|
\le L_{i,m}\|o_i-\bar{o}_i\|,
\quad \forall o_i,\bar{o}_i\in\mathcal{O}.
\label{eq:lipschitz_h}
\end{equation}
\end{assumption}

\begin{assumption}[Observation sensitivity to tracking error]
\label{ass:obs_lipschitz}
There exists a constant $L_{O,i}\ge 0$ such that for all $t\ge 0$,
\begin{equation}
\|o_i(t)-o_{i,d}(t)\| \le L_{O,i}\,\|p_i(t)-p_{i,d}(t)\|.
\label{eq:obs_lipschitz}
\end{equation}
\end{assumption}

\begin{definition}[Tracking-error--tightened constraints]
\label{def:tightened_constraints}
Define the tightened constraint functions
\begin{equation}
h_{i,\mathrm{tight}}^{(m)}(o_i)
:= h_i^{(m)}(o_i) + L_{i,m}\,L_{O,i}\,\varepsilon_{\mathrm{trk}} .
\label{eq:tightened_constraint}
\end{equation}
\end{definition}
We now show that, under bounded tracking error and mild regularity assumptions,
DGCBF-certified policies ensure continuous-time safety.
\begin{theorem}[Continuous-time safety under hierarchical tracking]
\label{thm:ct_safety_from_dgcbf}
Suppose Assumptions~\ref{ass:tracking_error}, \ref{ass:lipschitz_constraints},
and \ref{ass:obs_lipschitz} hold. Let $t_k := k\Delta t$ and let $o_{i,d}^k := o_{i,d}(t)$ for
$t\in[t_k,t_{k+1})$, where $o_{i,d}(t)$ is formed by replacing $p_i(t)$ with
$p_{i,d}(t)$ in $o_i(t)$.
Assume the distributed policy $\pi$ renders the tightened constraints
\eqref{eq:tightened_constraint} forward invariant at decision steps on the
reference observations, i.e., for all $k$,
\begin{equation}
h_{i,\mathrm{tight}}^{(m)}(o_{i,d}^k)\le 0
~\Rightarrow~
h_{i,\mathrm{tight}}^{(m)}(o_{i,d}^{k+1})\le 0 ,
\quad \forall i,m.
\label{eq:disc_invariance_tight_ref}
\end{equation}
Then, if $h_{i,\mathrm{tight}}^{(m)}(o_{i,d}^0)\le 0$ for all $i,m$, the original
continuous-time constraints hold for all time:
\begin{equation}
h_i^{(m)}(o_i(t)) \le 0,
\quad \forall t\ge 0,~\forall i,~\forall m.
\label{eq:ct_safety_result}
\end{equation}
\end{theorem}

\begin{proof}
By \eqref{eq:disc_invariance_tight_ref} and the initial condition
$h_{i,\mathrm{tight}}^{(m)}(o_{i,d}^0)\le 0$, we have
\begin{equation}
h_{i,\mathrm{tight}}^{(m)}(o_{i,d}^k)\le 0,
\quad \forall k\in\mathbb{N},~\forall i,m.
\label{eq:tight_ref_allk}
\end{equation}
Fix any $t\ge 0$, and let $k$ be such that $t\in[t_k,t_{k+1})$.
Under Assumption~\ref{ass:tracking_error}, the reference is constant on each
decision interval, i.e., $p_{i,d}(t)=p_{i,d}^k$ for $t\in[t_k,t_{k+1})$, and thus
$o_{i,d}(t)=o_{i,d}^k$ on the same interval. Therefore,
\begin{equation}
h_{i,\mathrm{tight}}^{(m)}(o_{i,d}(t))
= h_{i,\mathrm{tight}}^{(m)}(o_{i,d}^k)\le 0.
\label{eq:tight_ref_allt}
\end{equation}
It follows from Definition~\ref{def:tightened_constraints} and
\eqref{eq:tight_ref_allt} that
\begin{equation}
h_i^{(m)}(o_{i,d}(t)) \le -L_{i,m}\,L_{O,i}\,\varepsilon_{\mathrm{trk}}.
\label{eq:h_ref_margin}
\end{equation}
By Assumption~\ref{ass:lipschitz_constraints},
\begin{equation}
h_i^{(m)}(o_i(t))
\le h_i^{(m)}(o_{i,d}(t)) + L_{i,m}\|o_i(t)-o_{i,d}(t)\|.
\label{eq:lipschitz_step}
\end{equation}
Using Assumption~\ref{ass:obs_lipschitz} and Assumption~\ref{ass:tracking_error},
\[
\|o_i(t)-o_{i,d}(t)\|
\le L_{O,i}\|p_i(t)-p_{i,d}(t)\|
\le L_{O,i}\varepsilon_{\mathrm{trk}}.
\]
Substituting this bound and \eqref{eq:h_ref_margin} into \eqref{eq:lipschitz_step}
gives
\[
h_i^{(m)}(o_i(t))
\le \big(-L_{i,m}L_{O,i}\varepsilon_{\mathrm{trk}}\big)
+ L_{i,m}L_{O,i}\varepsilon_{\mathrm{trk}}
=0.
\]
Therefore $h_i^{(m)}(o_i(t))\le 0$ for all $t\ge 0$, proving
\eqref{eq:ct_safety_result}.
\end{proof}

\section{Experiments}

\begin{figure}[!t]
    \centering
    \includegraphics[width=0.90\columnwidth]{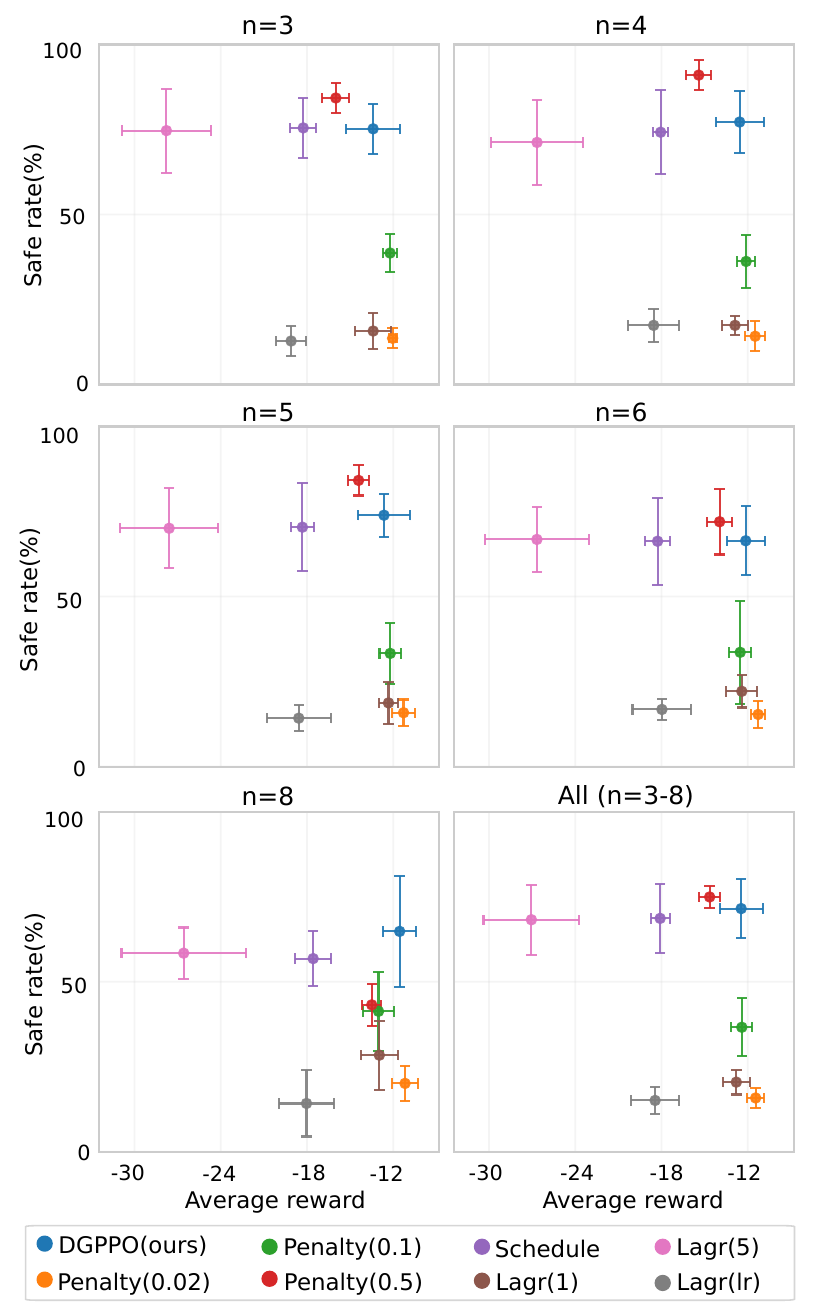}
    \vspace{-0.2cm}
    \caption{\textbf{Simulation results across different team sizes.}
    For each number of agents ($n \in \{3,4,5,6,8\}$), we evaluate policies on $25$ randomly generated scenarios.
    Each marker reports the mean reward and safety rate across five evaluation seeds (125 scenarios total), with error bars showing variability across seeds. We compare several distributed policy optimization backbones under the same setup.}
    \vspace{-0.4cm}
    \label{fig:sim_results}
\end{figure}

\subsection{Simulation Experiments}
We evaluate our approach in a simulation environment based on the proposed
2D abstraction of the multi-agent cooperative transport task.
In each episode, $N \in \{3,4,5\}$ agents (radius $r = 0.09\,\si{m}$, mass $0.027\,\si{kg}$) collaboratively translate and rotate a rigid $N$-gon payload (mass $0.045\,\si{kg}$) toward a target pose, while circular obstacles are randomly sampled in the environment. The payload mass is fixed, while cable stiffness is randomized to capture variations in the dynamics.
System dynamics are integrated using a JAX-based physics engine with timestep $\Delta t = 0.1\,\si{s}$ and five substeps per decision step.

Each agent is equipped with a $32$-beam planar LiDAR sensor with a maximum range
of $0.5\,\si{m}$ and communicates with neighboring agents within a radius of
$0.25\,\si{m}$. Observations include the agent’s own kinematic state, the payload
pose and velocity, the relative goal, and the shortest-$8$ LiDAR returns. To model
sensing uncertainty, we inject uniform noise of $\pm 5\%$ on positions and
orientations and $\pm 30\%$ on velocities.

Policies were trained for $10^5$ steps using $128$ parallel environments in a vectorized JAX-based simulator. Training required approximately $12$ hours on a single workstation equipped with an Intel Core Ultra~9~285K CPU, $62\,\si{GiB}$ of RAM, and an NVIDIA GeForce RTX~5090 GPU.

\begin{figure*}[!t]
    \centering

    \begin{overpic}[width=0.95\textwidth]{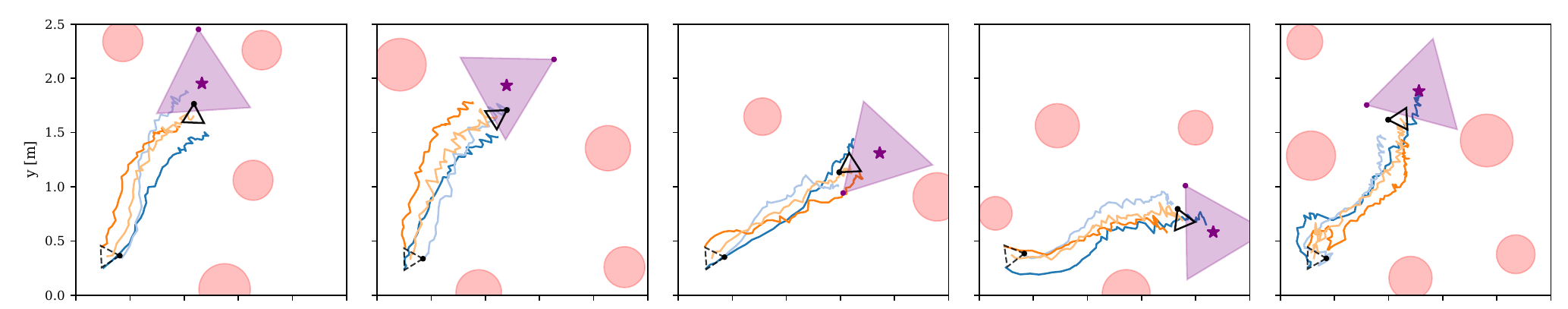}
        \put(-0.5,18){\normalsize\textbf{(a)}}
    \end{overpic}
    \vspace{-0.2cm}
    \begin{overpic}[width=0.95\textwidth]{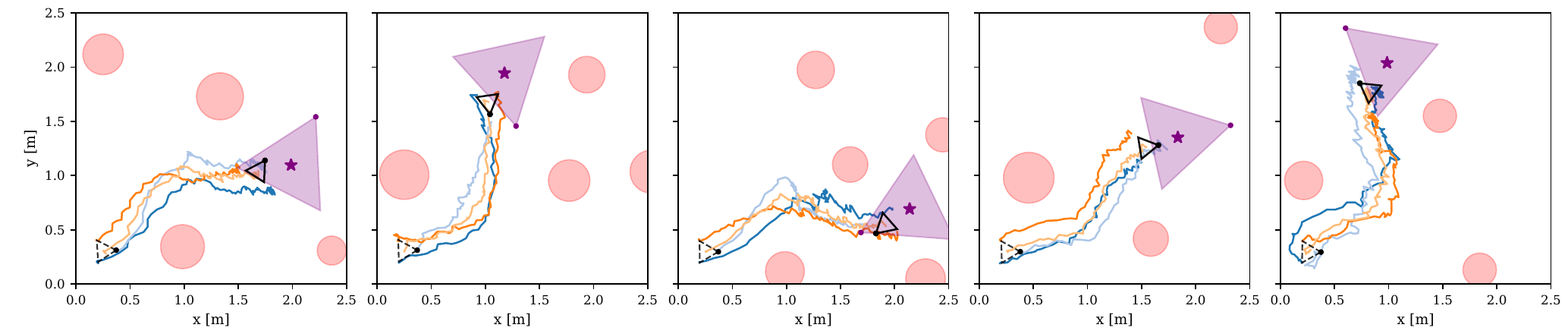}
        \put(-0.5,20){\normalsize\textbf{(b)}}
    \end{overpic}
    \vspace{-0.2cm}
    \begin{overpic}[width=0.95\textwidth]{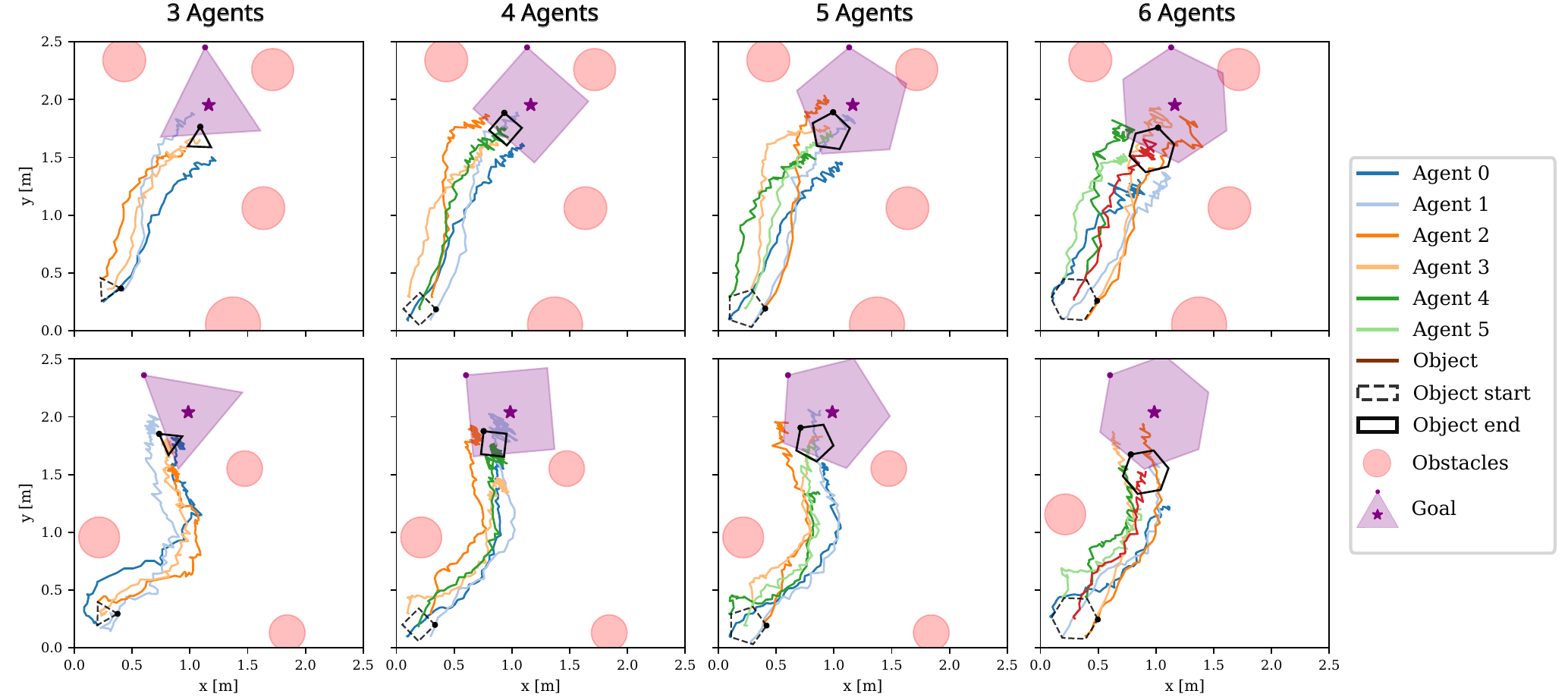}
        \put(-0.5,43){\normalsize\textbf{(c)}}
    \end{overpic}
    \caption{\textbf{Hardware validation of distributed multi-drone transportation.}
    (\textbf{a}) Easy scenarios where obstacles do not significantly affect the payload trajectory.
    (\textbf{b}) Hard scenarios where obstacles actively constrain the transport trajectories.
    (\textbf{c}) Scalability and generalization of a single learned policy evaluated with $3$–$6$ agents.}
    \vspace{-0.4cm}
    \label{fig:hardware_results}
\end{figure*}

\subsubsection{Reward and Safety Constraints Design}

\textbf{Reward.}
The reward function encourages the payload to reach the target pose while
promoting smooth and low-effort control actions. At each timestep $t$, the reward
is defined as
\begin{equation}
\begin{aligned}
r_t
= {} &
- w_d \, \| p_p - p_g \|_2^2
- w_\theta \, \left| \theta_p - \theta_g \right| \\
& - w_{\mathrm{thr}} \, \mathds{1}\!\left[ \| p_p - p_g \|_2 > d_{\mathrm{thr}} \right] \\
& - w_u \, \frac{1}{N} \sum_{i=1}^{N} \| u_i \|_2^2
- w_s \, \frac{1}{N} \sum_{i=1}^{N} \| u_i - u_i^{\mathrm{prev}} \|_2^2 ,
\end{aligned}
\label{eq:reward}
\end{equation}
where $p_p$ and $p_g$ represent the payload and goal positions, $\theta_p$ and
$\theta_g$ their corresponding orientations, $u_i$ the planar control action of
agent~$i$, and $u_i^{\mathrm{prev}}$ the control action at the previous time step. The weights are set to $w_d = 0.06$, $w_\theta = 0.06$, $w_{\mathrm{thr}} = 0.001$, $w_u = 0.1$, and $w_s = 0.1$. We optimize policies by maximizing the cumulative reward,
which is equivalent to minimizing a task cost defined as the negative reward,
consistent with the formulation in~\eqref{eq:problem_formulation}.

\textbf{Safety Constraints.}
In addition to the reward, we compute per-agent constraints that quantify
safety. These constraints are used for reporting and
for learning distributed safety certificates. The per-agent constraint function for agent $i$ is defined as
\begin{equation}
h_i =
\begin{bmatrix}
4\bigl(2r - \min_{j \neq i} \|p_i - p_j\|_2\bigr) \\
2\bigl(r - \min_k \|p_i - o_k\|_2\bigr) \\
2\,d_{\mathrm{po}} \\
10\bigl(\|p_i - p_{v,i}\|_2 - d_{\mathrm{tol}}\bigr)
\end{bmatrix},
\end{equation}
where $p_i$ denotes the position of agent $i$, $o_k$ denotes the position of the
$k$-th obstacle, and $p_{v,i}$ is the position of the attachment vertex assigned to agent~$i$. The scalar $d_{\mathrm{po}}$ denotes the payload--obstacle distance, and
$d_{\mathrm{tol}}$ specifies the allowable deviation between an agent and its
assigned attachment vertex. Each constraint component is margin-shifted by $\varepsilon = 0.5$ and clipped to the range $[-1,1]$.

\subsubsection{Evaluation Metrics}
We evaluate policy performance using two metrics: average reward and safety rate. For evaluation, we generate multiple randomized scenarios in
which the initial and goal payload positions and orientations, as well as obstacle
sizes and locations, are randomly sampled. For each team size
$n \in \{3,4,5,6,8\}$, we evaluate the policy on $25$ distinct scenarios. The average reward is computed over the full episode, while the safety rate is defined as the fraction of agents that satisfy all safety constraints throughout the execution.
\subsubsection{Policy Optimization Backbones}
We compare representative policy optimization backbones for our
framework and adopt DGPPO based on the comparison. Specifically, we considered the MARL algorithm InforMARL~\cite{nayak2023scalable} and the safe MARL algorithm MAPPO-Lagrangian~\cite{gu2021multi,gu2023safe}, which were trained and tested under the same environmental settings. \rev{For InforMARL}\del{For penalty-based methods}, constraint violations are aggregated as
$\textit{Penalty}(\beta)
= \beta \sum_{i=1}^{N} \sum_{m=1}^{M} \max(0, h_i^{(m)})$
and added to the cost function. We additionally evaluated a weight-scheduling scheme (\textit{Schedule}), where $\beta$ was initialized to $0.01$ and increased at $50\%$ and $75\%$ of the total training steps. For MAPPO-Lagrangian, we employed a GNN-based policy backbone to ensure architectural comparability across backbones. We evaluated two different initializations of the Lagrange multiplier with a learning rate of $10^{-7}$, denoted as \textit{Lagr}($\lambda_0$), and additionally increased the learning rate to $10^{-4}$ (\textit{Lagr(lr)}) with $\lambda_0 = 1$. All methods were trained for the same number of update steps, selected to be sufficiently large to ensure convergence across all approaches.

\subsubsection{Backbone Comparison}
\Cref{fig:sim_results} compares the backbones across team sizes. \rev{DGPPO achieves a favorable reward--safety trade-off: it maintains high safety rates while preserving competitive task reward compared with the Lagrangian and penalty-based variants. This trend remains stable as the number of agents increases from the training range of $3$--$5$ to larger teams such as $6$ and $8$, indicating scalability beyond the training team sizes.} While DGPPO and the $\beta=0.5$ penalty variant are similar in aggregate, under a stricter $0.1\,\si{m}$ goal threshold DGPPO outperforms at every team size, with an average hit rate of $84\%$ vs.\ $59\%$ for the penalty variant. This justifies DGPPO as the policy optimization backbone, while the proposed framework remains modular and could incorporate other safe RL methods.

\subsubsection{Sensing and Communication Robustness}
We evaluate the test-time robustness of the trained DGPPO policy on
teams of $n=5$ agents under per-agent message dropout and additional
Gaussian LiDAR noise, with 100 episodes per condition. The nominal
safety rate is $76.2\%$, which decreases to $72.4\%$ under $2\%$
dropout and $71.4\%$ under $\sigma = 0.03\,\si{m}$ LiDAR noise, and
further drops to $61.4\%$ under $5\%$ dropout and $52.6\%$ under
$\sigma = 0.045\,\si{m}$ LiDAR noise.

\begin{figure*}[!t]
    \centering
    \includegraphics[width=0.90\textwidth]{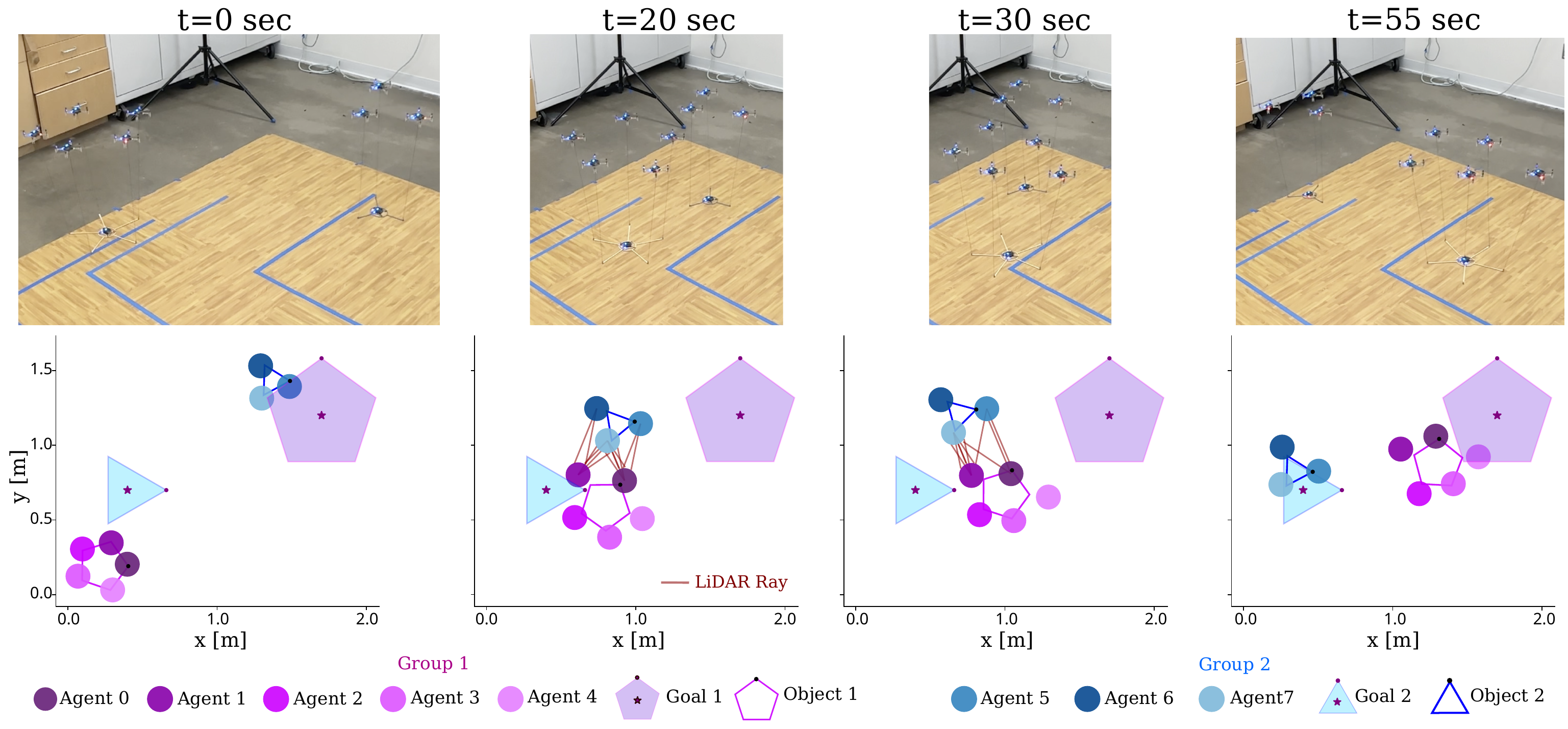}
    \vspace{-0.2cm}
    \caption{\textbf{Multi-group multi-drone transportation experiment.}
    Time-lapse snapshots and recorded trajectories show two independent teams transporting separate payloads while perceiving the other team as dynamic obstacles via LiDAR.}
    \vspace{-0.4cm}
    \label{fig:multigroup}
\end{figure*}
\subsection{Hardware Experiments}
We deploy the learned policies on a real-world Crazyflie~2.1 quadrotor platform equipped with the thrust-upgrade bundle to provide sufficient thrust margin for stable hovering and payload transport.
In all experiments, the drones are uniformly distributed around the payload with an inter-agent spacing of $0.20\,\si{m}$. The payloads are either 3D-printed or laser-cut rigid polygons.
To localize the payload, an additional Crazyflie unit without motors is mounted at the payload’s center of mass.
All experiments use cables of length $0.50\,\si{m}$, and the payload mass is approximately $27\,\si{g}$.

During hardware experiments, a centralized computer is used for state
estimation and logging, while decentralized execution is emulated by providing local observation to each agent's policy. The policy outputs planar accelerations in the $x$ and $y$ directions at a rate of $10\,\mathrm{Hz}$. These accelerations are integrated to generate the next desired planar waypoint. A low-level onboard PID controller then tracks this waypoint. The commanded accelerations are scaled by a factor of $15$. This scaling was selected empirically and was found to yield stable closed-loop behavior under the hierarchical control architecture. The desired altitude ($z$-axis) is held constant throughout the experiment. Across all hardware trials, the mean position tracking error remained below 5.8\,\si{cm}, with most experiments exhibiting errors between 3–5\,\si{cm}.

To validate policy performance, we design $10$ hardware test scenarios, comprising $5$ \emph{Easy} scenarios (Fig.~\ref{fig:hardware_results}a) and $5$ \emph{Hard} scenarios (Fig.~\ref{fig:hardware_results}b).
In each scenario, the initial positions of the drones and payload are fixed, while the goal pose and orientation are randomly sampled such that the distance between the initial and goal payload positions is $2.0\,\mathrm{m}$. Circular obstacles are randomly placed within the workspace.
In the Easy scenarios, no obstacles lie directly between the initial and goal positions, such that obstacle avoidance does not significantly alter the nominal transport trajectory. In contrast, the Hard scenarios contain obstacles that obstruct the direct path, requiring nontrivial maneuvering. In all tested scenarios, teams of three drones successfully transport the payload to the desired goal position and orientation. Task completion is declared when the payload is within $0.30\,\si{m}$ of the goal position and the orientation error is less than $0.2\,\si{rad}$.

We further evaluate the scalability and generalization of the learned policy to varying team sizes (\Cref{fig:hardware_results}c). Specifically, we test the first Easy scenario and the final Hard scenario using teams of $4$, $5$, and $6$ drones. Notably, although domain randomization during training was restricted to team sizes in the range $3$--$5$, the learned policy generalizes reliably to the unseen $6$-drone configuration. In the $6$-drone Hard scenario, the obstacle density is intentionally reduced to accommodate the larger formation footprint. \rev{This modified trial tests zero-shot detouring with six agents and is not directly comparable to the smaller-team Hard results; the controlled cross-team comparison is provided in simulation (\Cref{fig:sim_results}).} All scenarios are successfully completed, except for one $6$-drone Hard scenario, where the payload reaches the target region but plateaus near the goal, resulting in a final position error of $0.5\,\si{m}$ and an orientation error of $0.2\,\si{rad}$. This suggests a terminal pose-refinement limitation from reward shaping, not a safety failure.

\subsection{Multi-Group Multi-Drone Transportation}

A key advantage of our approach is that it does not assume static obstacles. Since obstacles are represented implicitly through LiDAR observations, the learned policy can naturally respond to dynamic obstacles without explicit modeling.

We evaluate this capability in a hardware experiment (\Cref{fig:Figure1}.c). Without any additional fine-tuning, we conduct multi-group tests involving two teams with different sizes, specifically a $5$-agent group and a $3$-agent group. As shown in fig.~\ref{fig:multigroup}, agents treat drones from the other group as dynamic obstacles using LiDAR sensing and successfully avoid inter-group collisions during execution (e.g., at $t=20\si{s}$ and $t=30\si{s}$). Both groups reliably transport their respective payloads to the target positions and orientations without interference, demonstrating the robustness, scalability, and generality of the proposed distributed policy. For this experiment, the communication radius is increased to $0.30\,\si{m}$ to ensure reliable inter-agent coordination.

\section{Conclusion}
We demonstrated scalable and safety-aware cooperative multi-drone payload transportation through extensive real-world experiments, including teams larger than those encountered during training and dynamic multi-group scenarios in which other drone teams act as moving obstacles. These results establish the practical feasibility of fully distributed, learning-based aerial collaboration in complex physical environments.

The successful zero-shot transfer from simulation to hardware provides empirical support for modeling cable-suspended transport using a minimal planar decision abstraction, and suggests that learned safety certificates with appropriate margins can be preserved during real-world execution.

This work is currently limited to planar payload transport with evenly
distributed cable attachment points. Future work will investigate broader attachment configurations, cable properties, and payload variations, as well as extend the framework to fully three-dimensional transportation.

\bibliographystyle{IEEEtran}
\bibliography{references.bib}
\end{document}